\documentclass{article} % For LaTeX2e
\usepackage{hyperref}
\usepackage{graphics}
\usepackage{amsthm}
\usepackage{mathtools}
\usepackage{amsmath,amsfonts,amssymb}       % blackboard math symbols
\usepackage{appendix}
\usepackage{bbm}
\usepackage{graphicx}
\usepackage{color}

\newcommand\EE{\mathbb{E}}

\newcommand\NN{\mathbb{N}}

\newcommand\PP{\mathbb{P}}

\newcommand\RR{\mathbb{R}}

\renewcommand{\d}[1]{\ensuremath{\operatorname{d}\!{#1}}}

\newtheorem{coro}{Corollary}
\newtheorem{theo}{Theorem}

\newtheorem{ass}{Assumption}
\theoremstyle{definition}
\newtheorem{exa}{Example}
\theoremstyle{definition}

\newcommand*\diff{\mathop{}\!\mathrm{d}}

\newcommand\manX{\mathcal{X}}
\newcommand\manF{\mathcal{F}}
\usepackage{appendix}
\usepackage{authblk}
\usepackage{algpseudocode}% http://ctan.org/pkg/algorithmicx
\usepackage{algorithm}% http://ctan.org/pkg/algorithm
\usepackage[letterpaper,margin=1.75in]{geometry}
\usepackage[font={small,it}]{caption}
\graphicspath{{./figures/}}
\usepackage{expl3}[2012-07-08]
\ExplSyntaxOn
\cs_new_eq:NN \fpeval \fp_eval:n
\ExplSyntaxOff
\title{Wasserstein GAN}
\author[1]{Martin Arjovsky}
\author[2]{Soumith Chintala}
\author[1,2]{L\'eon Bottou}
\affil[1]{Courant Institute of Mathematical Sciences}
\affil[2]{Facebook AI Research}

\date{}

\begin{document}
\maketitle

%%%%%%%%%%%%%%%%%%%%%%%%%%%%%%%%%%%%%%%%
\section{Introduction}
The problem this paper is concerned with is that of unsupervised learning.
Mainly, what does it mean to learn a probability
distribution? The classical answer to this is to learn a probability
density. This is often done by defining a parametric family of densities
$(P_\theta)_{\theta \in \RR^d}$ and finding the one that maximized the likelihood on our data:
if we have real data examples $\{x^{(i)}\}_{i=1}^m$, we would solve the problem
\begin{equation*}
\max_{\theta \in \RR^d} \frac{1}{m}\sum_{i=1}^m \log P_\theta(x^{(i)})
\end{equation*}
If the real data distribution $\PP_r$ admits a density and $\PP_\theta$ is the
distribution of the parametrized density $P_\theta$, then, asymptotically, this
amounts to minimizing the Kullback-Leibler divergence $KL(\PP_r \| \PP_\theta)$.

For this to make sense, we need the model density $P_\theta$ to exist.
This is not the case in the rather common situation where we are
dealing with distributions supported by low dimensional manifolds.
It is then unlikely that the model manifold and the true distribution's support
have a non-negligible intersection (see \cite{Arj-ea-Princ}),
and this means that the KL distance is not defined (or simply infinite).

The typical remedy is to add a noise term to the model
distribution. This is why virtually all generative models described in
the classical machine learning literature include a noise
component. In the simplest case, one assumes a Gaussian noise with
relatively high bandwidth in order to cover all the examples. It is
well known, for instance, that in the case of image generation models,
this noise degrades the quality of the samples and makes them
blurry.
For example, we can see in the recent paper \cite{Wu-ea-AIS}
that the optimal standard deviation of the noise added to the model
when maximizing likelihood is around 0.1 to each pixel in a generated image,
when the pixels were already normalized to be in the range $[0, 1]$. This is
a very high amount of noise, so much that when papers report the samples of
their models, they don't add the noise term on which they report likelihood numbers.
In other words, the added noise term
is clearly incorrect for the problem, but is needed to make the
maximum likelihood approach work.

Rather than estimating the density of $\PP_r$ which may not exist, we
can define a random variable $Z$ with a fixed distribution $p(z)$ and
pass it through a parametric function $g_\theta: \mathcal{Z} \rightarrow
\mathcal{X}$ (typically a neural network of some kind) that directly
generates samples following a certain distribution $\PP_\theta$. By
varying $\theta$, we can change this distribution and make it close to
the real data distribution $\PP_r$.  This is useful in two ways. First
of all, unlike densities, this approach can represent distributions
confined to a low dimensional manifold.  Second, the ability to easily
generate samples is often more useful than knowing the numerical value
of the density (for example in image superresolution or semantic
segmentation when considering the conditional distribution of the
output image given the input image). In general, it is computationally
difficult to generate samples given an arbitrary high dimensional
density \cite{Neal-AIS}.

Variational Auto-Encoders (VAEs) \cite{Kingma-ea-VAE} and Generative
Adversarial Networks (GANs) \cite{Good-ea-GAN} are well known examples
of this approach.  Because VAEs focus on the approximate likelihood of
the examples, they share the limitation of the standard models and
need to fiddle with additional noise terms. GANs offer much more
flexibility in the definition of the objective function, including
Jensen-Shannon \cite{Good-ea-GAN}, and all $f$-divergences
\cite{Nowo-ea-fgan} as well as some exotic combinations
\cite{Ferenc15}. On the other hand, training GANs is well known for
being delicate and unstable, for reasons theoretically investigated in
\cite{Arj-ea-Princ}.

In this paper, we direct our attention on the various ways to measure
how close the model distribution and the real distribution are, or
equivalently, on the various ways to define a distance or divergence
$\rho(\PP_\theta,\PP_r)$. The most fundamental difference between such
distances is their impact on the convergence of sequences of
probability distributions. A sequence of distributions
$(\PP_t)_{t\in\NN}$ converges if and only if there is a distribution
$\PP_\infty$ such that $\rho(\PP_t,\PP_\infty)$ tends to zero,
something that depends on how exactly the distance $\rho$ is defined.
Informally, a distance $\rho$ induces a weaker topology when it makes
it easier for a sequence of distribution to converge.\footnote{More
  exactly, the topology induced by $\rho$ is weaker than that induced
  by $\rho'$ when the set of convergent sequences under $\rho$ is a
  superset of that under $\rho'$.}  Section~\ref{sec-distances}
clarifies how popular probability distances differ in that respect.
%  Continuity means
%that when a sequence of parameters $(\theta_t)_{t\in\NN}$ converges to
%$\theta$, the corresponding sequence of distributions
%$(\PP_{\theta_t})_{t\in\NN}$ also converges to $\PP_\theta$. When the
%topology induced by $\rho$ is weak, it is easier to ensure that
%$(\PP_{\theta_t})_{t\in\NN}$ converges, and therefore it is easier to
%define a parametric model that is both interesting and amenable to
%optimization. Since the Earth Mover (EM) distance \cite{Vil}, also
%called Wasserstein(1) distance, is one of the weakest distances
%between probability distribution, it is only natural to seek models
%that optimize loss functions closely related to EM.

In order to optimize the parameter $\theta$, it is of course desirable
to define our model distribution $\PP_\theta$ in a manner that makes
the mapping $\theta\mapsto\PP_\theta$ continuous.
Continuity means that when a sequence of parameters $\theta_t$ converges
to $\theta$, the distributions $\PP_{\theta_t}$ also converge
to $\PP_\theta$. However, it is essential to remember that the
notion of the convergence of the distributions $\PP_{\theta_t}$ depends
on the way we compute the distance between distributions.
The weaker this distance, the easier
it is to define a continuous mapping from $\theta$-space to
$\PP_{\theta}$-space, since it's easier for the
distributions to converge.
The main reason we care about the mapping $\theta \mapsto \PP_\theta$
to be continuous is as follows. If $\rho$ is our notion of
distance between two distributions, we would like to have a loss
function $\theta \mapsto \rho(\PP_\theta, \PP_r)$ that is continuous,
and this is equivalent to having the mapping $\theta \mapsto \PP_\theta$
be continuous when using the distance between distributions $\rho$.

\pagebreak
The contributions of this paper are:
\begin{itemize}
\item In Section~\ref{sec-distances}, we provide a comprehensive
  theoretical analysis of how the Earth Mover (EM) distance behaves in comparison to
  popular probability distances and divergences used in the
  context of learning distributions.
\item In Section~\ref{sec-wgan}, we define a form of GAN called Wasserstein-GAN
  that minimizes a reasonable and efficient approximation of the EM distance,
  and we theoretically show that the corresponding optimization 
  problem is sound.
\item In Section~\ref{sec-experiments}, we empirically show that WGANs
  cure the main training problems of GANs. In particular, training
  WGANs does not require maintaining a careful balance in training of
  the discriminator and the generator, and does not require a careful
  design of the network architecture either. The mode dropping
  phenomenon that is typical in GANs is also drastically reduced.  One
  of the most compelling practical benefits of WGANs is the ability to
  continuously estimate the EM distance by training the
  discriminator to optimality. Plotting these learning curves is not
  only useful for debugging and hyperparameter searches, but also
  correlate remarkably well with the observed sample quality.
\end{itemize}

%This paper is organized as follows. First we discuss more precisely the
%notion of weak and strong topologies on probability distributions, together
%with illustrative examples. We take advantage of those tools to better understand
%and explain the typical GAN training problems. We show how the Wasserstein
%distance eliminates some of them, and then establish a relationship
%between Wasserstein, KL and the JSD divergences. Finally,
%we introduce the WGAN algorithm, present
%empirical results, and discuss related works.

%%%%%%%%%%%%%%%%%%%%%%%%%%%%%%%%%%%%%%%%
\section{Different Distances}
\label{sec-distances}

We now introduce our notation. Let $\manX$ be a compact metric set
(such as the space of images $[0,1]^d$) and let $\Sigma$ denote the
set of all the Borel subsets of $\manX$. Let $\text{Prob}(\manX)$
denote the space of probability measures defined on $\manX$.
We can now define elementary distances and divergences
between two distributions $\PP_r,\PP_g\in\text{Prob}(\manX)$:
\begin{itemize}
\item
  The \emph{Total Variation} (TV) distance
  \[ \delta(\PP_r,\PP_g) = \sup_{A\in\Sigma} |\PP_r(A)-\PP_g(A)|~. \]
\item
  The \emph{Kullback-Leibler} (KL) divergence
  \[ KL(\PP_r\|\PP_g) = \int \log\left(\frac{P_r(x)}{P_g(x)}\right) P_r(x) d\mu(x)~, \]
  where both $\PP_r$ and $\PP_g$ are assumed to be absolutely continuous,
  and therefore admit densities, with respect to a same measure $\mu$ defined on $\manX$.\footnote{%
    Recall that a probability distribution
    $\PP_r\in\text{Prob}(\manX)$ admits a density $p_r(x)$ with
    respect to $\mu$, that is, $\forall A\in\Sigma$, $\PP_r(A) =
    \int_A P_r(x) d\mu(x)$, if and only it is absolutely continuous
    with respect to $\mu$, that is, $\forall A\in\Sigma$,
    $\mu(A)=0\Rightarrow\PP_r(A)=0$~.}
  The KL divergence is famously assymetric and possibly infinite when there
  are points such that $P_g(x)=0$ and $P_r(x)>0$.
\item
  The \emph{Jensen-Shannon} (JS) divergence
  \[ JS(\PP_r,\PP_g) = KL(\PP_r\|\PP_m)+KL(\PP_g\|\PP_m)~, \]
  where $\PP_m$ is the mixture $(\PP_r+\PP_g)/2$. This divergence
  is symmetrical and always defined because we can choose $\mu=\PP_m$.
\item
  The \emph{Earth-Mover} (EM) distance or Wasserstein-1
  \begin{equation}
    W(\PP_r, \PP_g) = \inf_{\gamma \in \Pi(\PP_r ,\PP_g)} \EE_{(x, y) \sim \gamma}\big[\:\|x - y\|\:\big]~,
    \label{eq::W}
  \end{equation}
  where $\Pi(\PP_r,\PP_g)$ denotes the set of all joint distributions
  $\gamma(x,y)$ whose marginals are respectively $\PP_r$ and
  $\PP_g$. Intuitively, $\gamma(x,y)$ indicates how much ``mass'' must
  be transported from $x$ to $y$ in order to transform the distributions
  $\PP_r$ into the distribution $\PP_g$. The EM distance then
  is the ``cost'' of the optimal transport plan.
\end{itemize}

The following example illustrates how apparently simple sequences of
probability distributions converge under the EM distance but do not
converge under the other distances and divergences defined above.

\begin{exa}[Learning parallel lines] \label{exa::lines}
Let $Z \sim U[0,1]$ the uniform distribution
on the unit interval. Let $\PP_0$ be the
distribution of $(0, Z) \in \RR^2$ (a 0 on the x-axis
and the random variable $Z$ on the y-axis), uniform on a straight
vertical line passing through the origin. Now
let $g_\theta(z) = (\theta, z)$ with $\theta$
a single real parameter. It is easy to see
that in this case,
\begin{itemize}
\item
  $W(\PP_0, \PP_\theta) = |\theta|$,
\item
  $\displaystyle JS(\PP_0,\PP_\theta) =
  \begin{cases}
    \log 2 &\quad \text{if } \theta \neq 0~, \\
    0 &\quad \text{if } \theta = 0~,
  \end{cases}$
\item
  $\displaystyle KL(\PP_\theta \| \PP_0) = KL(\PP_0 \| \PP_\theta) =
  \begin{cases}
    +\infty &\quad \text{if } \theta \neq 0~, \\
    0 &\quad \text{if } \theta = 0~,
  \end{cases}$
\item
  and $\displaystyle \delta(\PP_0,\PP_\theta) = 
  \begin{cases}
    1 &\quad \text{if } \theta \neq 0~, \\
    0 &\quad \text{if } \theta = 0~.
  \end{cases}$
\end{itemize}
When $\theta_t\rightarrow0$,
the sequence $(\PP_{\theta_t})_{t\in\NN}$ converges
to $\PP_0$ under the EM distance, but does not converge
at all under either the JS, KL, reverse KL, or TV divergences.
\autoref{fig::lc} illustrates this for the case of the EM and JS distances.

\begin{figure}[h!]
  \centering
  \begin{minipage}[b]{0.5\linewidth}
    \centering
    \includegraphics[scale=0.325]{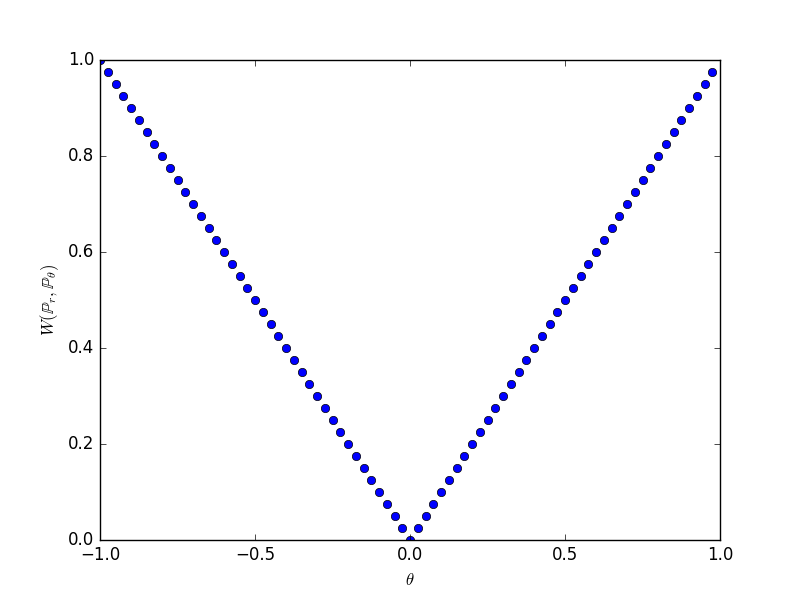} %0.325
  \end{minipage}%%
  \begin{minipage}[b]{0.5\linewidth}
    \centering
    \includegraphics[scale=0.325]{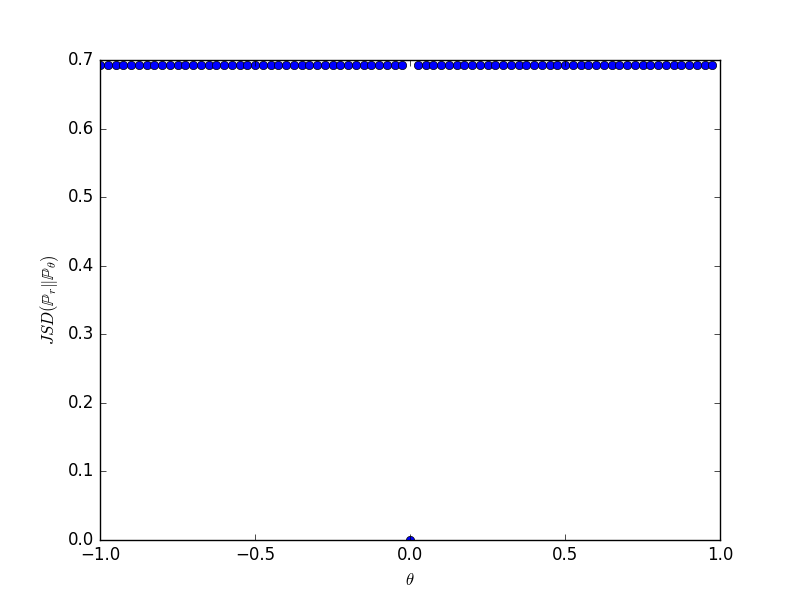}
  \end{minipage}
%  \includegraphics[width=.3\linewidth]{wloss.png}
%  \qquad
%  \includegraphics[width=.3\linewidth]{jsdloss.png}
  \caption{These plots show $\rho(\PP_\theta,\PP_0)$ as a function of $\theta$
    when $\rho$ is the EM distance (left plot) or the JS divergence (right plot).
    The EM plot is continuous and provides a usable gradient everywhere.
    The JS plot is not continuous and does not provide a usable gradient.}
  \label{fig::lc}
\end{figure}
\end{exa}

Example \autoref{exa::lines} gives us a case where we can learn a
probability distribution over a low dimensional manifold by doing
gradient descent on the EM distance. This cannot be done with the
other distances and divergences because the resulting loss function is
not even continuous. Although this simple example features
distributions with disjoint supports, the same conclusion holds when
the supports have a non empty intersection contained in a set of
measure zero. This happens to be the case when two low dimensional
manifolds intersect in general position~\cite{Arj-ea-Princ}.

Since the Wasserstein distance is much weaker than the JS distance\footnote{
The argument for \emph{why} this happens, and indeed
how we arrived to the idea that Wasserstein is what
we should really be optimizing is displayed in Appendix
\autoref{app::weak}. We strongly encourage the interested
reader who is not afraid of the mathematics to go through it.},
we can now ask whether $W(\PP_r, \PP_\theta)$ is
a continuous loss function on $\theta$ 
under mild assumptions. This, and more, is
true, as we now state and prove.

\begin{theo} \label{theo::regcost} Let $\PP_r$ be a fixed distribution over $\mathcal{X}$.
Let $Z$ be a random variable (e.g Gaussian) over another
space $\mathcal{Z}$. Let $g: \mathcal{Z} \times \RR^d \rightarrow \mathcal{X}$
be a function, that will be denoted $g_\theta(z)$ with $z$ the first coordinate
and $\theta$ the second. Let $\PP_\theta$ denote the distribution of $g_\theta(Z)$.
Then,
\begin{enumerate}
\item If $g$ is continuous in $\theta$, so is $W(\PP_r, \PP_\theta)$.
\item If $g$ is locally Lipschitz and satisfies regularity
assumption \ref{ass::lip},
then $W(\PP_r, \PP_\theta)$ is continuous
everywhere, and differentiable almost everywhere.
\item Statements 1-2 are false for the Jensen-Shannon divergence $JS(\PP_r, \PP_\theta)$
and all the KLs.
\end{enumerate}
\end{theo}
\begin{proof} 
See Appendix \ref{app::proofs}
\end{proof}

The following corollary tells us that learning by minimizing
the EM distance makes sense (at least in theory) with neural networks.
\begin{coro} \label{coro::nnregcost}
Let $g_\theta$ be any feedforward neural network\footnote{By a
feedforward neural network we mean a function composed by
affine transformations and pointwise nonlinearities which are
smooth Lipschitz
functions (such as the sigmoid, tanh, elu, softplus, etc).
\textbf{Note}: the statement is also true for rectifier
nonlinearities but the proof is more
technical (even though very similar) so we omit it.} parameterized by $\theta$, 
and $p(z)$ a prior over $z$ such that $\EE_{z \sim p(z)}[\|z\|] < \infty$ (e.g.
Gaussian, uniform, etc.). Then assumption
\ref{ass::lip} is satisfied and therefore $W(\PP_r, \PP_\theta)$
is continuous everywhere and differentiable almost everywhere.
\end{coro}
\begin{proof}
See Appendix \ref{app::proofs}
\end{proof}

All this shows that EM is a much more sensible
cost function for our problem than at least the Jensen-Shannon
divergence. The following theorem describes the relative strength of
the topologies induced by these distances and divergences, with KL the strongest,
followed by JS and TV, and EM the weakest.

\begin{theo} \label{theo::dist}
Let $\PP$ be a distribution on a compact space $\manX$ and 
$(\PP_n)_{n \in \NN}$ be a sequence
of distributions on $\manX$. Then, considering
all limits as $n \to \infty$,
\begin{enumerate}
\item The following statements are equivalent
\begin{itemize}
  \item $\delta(\PP_n, \PP) \to 0$
with $\delta$ the total variation distance.
  \item $JS(\PP_n,\PP) \to 0$ with
$JS$ the Jensen-Shannon divergence.
\end{itemize}
\item The following statements are equivalent
\begin{itemize}
  \item $W(\PP_n, \PP) \to 0$.
  \item $\PP_n \xrightarrow{\mathcal{D}} \PP$ where $\xrightarrow{\mathcal{D}}$ represents
convergence in distribution for random variables.
\end{itemize}
\item $KL(\PP_n \| \PP) \to 0$ or $KL(\PP \| \PP_n) \to 0$ imply
the statements in (1). 
\item The statements in (1) imply the statements in (2).
\end{enumerate}
\end{theo}
\begin{proof}
See Appendix \ref{app::proofs}
\end{proof}

This highlights the fact that 
the KL, JS, and TV distances are not sensible
cost functions when learning distributions
supported by low dimensional manifolds.
However the EM distance is sensible
in that setup. This obviously leads us to the next section
where we introduce a practical approximation
of optimizing the EM distance.

\section{Wasserstein GAN}
\label{sec-wgan}

Again, Theorem \ref{theo::dist} points to the fact that
$W(\PP_r, \PP_\theta)$ might have nicer properties
when optimized than $JS(\PP_r,\PP_\theta)$.
However, the infimum in
\eqref{eq::W} is highly intractable. On the other hand,
the Kantorovich-Rubinstein duality \cite{Vil}
tells us that
\begin{equation} \label{eq::KR}
W(\PP_r, \PP_\theta) = \sup_{\|f\|_L \leq 1} \EE_{x \sim \PP_r}
[f(x)] - \EE_{x \sim \PP_\theta}[f(x)]
\end{equation}
where the supremum is over all the 1-Lipschitz functions
$f: \manX \rightarrow \RR$. Note that if we replace $\|f\|_L \leq 1$
for $\|f\|_L \leq K$ (consider $K$-Lipschitz for some constant $K$), then
we end up with $K \cdot W(\PP_r, \PP_g)$. Therefore, if we have a
parameterized family of functions $\{f_w\}_{w \in \mathcal{W}}$
that are all $K$-Lipschitz for some $K$, we could consider
solving the problem
\begin{equation} \label{eq::wgan}
\max_{w \in \mathcal{W}} \EE_{x \sim \PP_r}[f_w(x)] -
\EE_{z \sim p(z)} [f_w(g_\theta(z)]
\end{equation}
and if the supremum in \eqref{eq::KR} is attained
for some $w \in \mathcal{W}$ (a pretty strong assumption
akin to what's assumed when proving consistency of an
estimator), this process would
yield a calculation of $W(\PP_r, \PP_\theta)$ up to
a multiplicative constant. Furthermore, we could consider
differentiating $W(\PP_r, \PP_\theta)$ (again, up to a constant)
by back-proping through equation \eqref{eq::KR} via
estimating $\EE_{z \sim p(z)}[\nabla_\theta f_w(g_\theta(z))]$.
While this is all intuition, we now prove that this process
is principled under the optimality assumption.
\begin{theo} \label{theo::gradW}
Let $\PP_r$ be any distribution. Let $\PP_\theta$ be
the distribution of $g_\theta(Z)$ with $Z$ a random
variable with density $p$ and $g_\theta$
a function satisfying assumption \ref{ass::lip}. 
Then, there is a solution $f: \manX \rightarrow \RR$
to the problem
$$ \max_{\|f\|_L \leq 1} \EE_{x \sim \PP_r}[f(x)] - 
  \EE_{x \sim \PP_\theta} [f(x)] $$
and we have
$$\nabla_\theta W(\PP_r, \PP_\theta)
= -\EE_{z \sim p(z)}[\nabla_\theta f(g_\theta(z))] $$
when both terms are well-defined.
\end{theo}
\begin{proof} See Appendix \autoref{app::proofs}
\end{proof}

Now comes the question of finding the function $f$ that
solves the maximization problem in equation \eqref{eq::KR}.
To roughly approximate
this, something that we can do is train a neural network
parameterized with weights $w$ lying in a compact space
$\mathcal{W}$ and then backprop through
$\EE_{z \sim p(z)}[\nabla_\theta f_w(g_\theta(z))]$, as we
would do with a typical GAN. Note that the fact that 
$\mathcal{W}$ is compact implies that all the functions
$f_w$ will be $K$-Lipschitz for some $K$ that only depends
on $\mathcal{W}$ and not the individual weights, therefore
approximating \eqref{eq::KR} up to an irrelevant scaling factor
and the capacity of the `critic' $f_w$. In order to have parameters $w$
lie in a compact space, something simple we can do is clamp
the weights to a fixed box (say $\mathcal{W} = [-0.01,0.01]^l$) after each
gradient update. The Wasserstein Generative Adversarial
Network (WGAN) procedure is described in Algorithm \autoref{algo::wgan}.

Weight clipping is a clearly terrible way to enforce a Lipschitz constraint.
If the clipping parameter is large, then it can take a long time
for any weights to reach their limit, thereby making it harder
to train the critic till optimality. If the clipping is small, this
can easily lead to vanishing gradients when the number of layers is big,
or batch normalization is not used (such as in RNNs). We experimented
with simple variants (such as projecting the weights to a sphere) with
little difference, and we stuck with weight clipping due to its simplicity
and already good performance. However, we do leave the topic of
enforcing Lipschitz constraints in a neural network setting for further
investigation, and we actively encourage interested researchers to improve
on this method.

\begin{algorithm}[t!]
  \caption{WGAN, our proposed algorithm. All experiments in the paper
used the default values $\alpha = 0.00005$, $c = 0.01$, $m=64$, $n_{\text{critic}}=5$.}\label{algo::wgan}
  \begin{algorithmic}[1]
    \Require: $\alpha$, the learning rate. $c$, the clipping parameter. $m$, the
      batch size. $n_{\text{critic}}$, the number of iterations of the critic
      per generator iteration.
    \Require: $w_0$, initial critic parameters. $\theta_0$,
      initial generator's parameters.
    \While{$\theta$ has not converged}
      \For{$t = 0, ..., n_{\text{critic}}$}
        \State Sample $\{x^{(i)}\}_{i=1}^m \sim \PP_r$ a batch from the real data.
        \State Sample $\{z^{(i)}\}_{i=1}^m \sim p(z)$ a batch of prior samples.
        \State $g_w \gets \nabla_w \left[\frac{1}{m}\sum_{i=1}^m f_w(x^{(i)}) - \frac{1}{m} \sum_{i=1}^m f_w(g_\theta(z^{(i)})) \right]$
        \State $w \gets w + \alpha \cdot \text{RMSProp}(w, g_w) $
        \State $w \gets \text{clip}(w, -c, c) $
      \EndFor
      \State Sample $\{z^{(i)}\}_{i=1}^m \sim p(z)$ a batch of prior samples.
      \State $g_\theta \gets -\nabla_\theta \frac{1}{m} \sum_{i=1}^m f_w(g_\theta(z^{(i)}))$ 
      \State $\theta \gets \theta - \alpha \cdot \text{RMSProp}(\theta, g_\theta)$
    \EndWhile
\end{algorithmic}
\end{algorithm}

The fact that the EM distance is continuous and differentiable a.e.
means that we can (and should) train
the critic till optimality. The argument is simple,
the more we train the critic, the more reliable gradient of
the Wasserstein we get, which is actually useful by the
fact that Wasserstein is differentiable almost everywhere.
For the JS, as the discriminator gets better the gradients get
more reliable but the true gradient is 0 since the JS is locally
saturated and we get vanishing gradients,
as can be seen in \autoref{fig::lc} of this paper
and Theorem 2.4 of \cite{Arj-ea-Princ}. In \autoref{fig::optdiscs}
we show a proof of concept of this, where we train
a GAN discriminator and a WGAN critic till optimality.
The discriminator learns very quickly to distinguish between
fake and real, and as expected provides no reliable gradient
information. The critic, however, can't saturate, and converges
to a linear function that gives remarkably clean gradients everywhere.
The fact that we constrain the weights limits the possible
growth of the function to be at most linear in different parts
of the space, forcing the optimal critic to have this behaviour.

\begin{figure}[ht]
    \centering
    \includegraphics[scale=0.5]{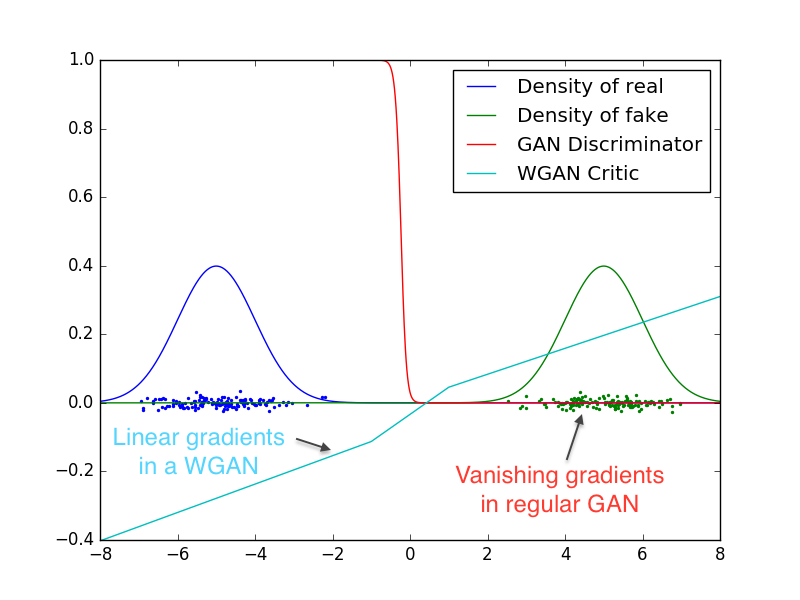}
    \caption{Optimal discriminator and critic when
learning to differentiate two Gaussians. As we can see,
the discriminator of a minimax GAN saturates and results
in vanishing gradients. Our WGAN critic provides very clean
gradients on all parts of the space.}
    \label{fig::optdiscs}
\end{figure}

Perhaps more importantly, the fact that we can train the critic
till optimality makes it impossible to collapse modes when we do.
This is due to the fact that mode collapse comes from the fact
that the optimal generator for a \emph{fixed} discriminator
is a sum of deltas on the points the discriminator assigns
the highest values, as observed by \cite{Good-ea-GAN} and
highlighted in \cite{Metz-ea-UG}.

In the following section we display the practical benefits
of our new algorithm, and we provide an in-depth comparison
of its behaviour and that of traditional GANs.

\section{Empirical Results}
\label{sec-experiments}

  We run experiments on image generation using our Wasserstein-GAN algorithm
and show that there are significant practical benefits to using it over the 
formulation used in standard GANs.

\medskip
\noindent
We claim two main benefits:
\begin{itemize}
\item a meaningful loss metric that correlates with the generator's 
convergence and sample quality
\item improved stability of the optimization process
\end{itemize}

\subsection{Experimental Procedure}

We run experiments on image generation. The target distribution to learn is the 
LSUN-Bedrooms dataset \cite{LSUN} -- a collection of natural images of
indoor bedrooms. Our baseline comparison is DCGAN \cite{Rad-ea-DCGAN},
a GAN with a convolutional architecture trained with the standard GAN procedure
using the $-\log D$ trick \cite{Good-ea-GAN}.
The generated samples are 3-channel images of 64x64 pixels in size.
We use the hyper-parameters specified in Algorithm \autoref{algo::wgan} for 
all of our experiments.

\subsection{Meaningful loss metric}

Because the WGAN algorithm attempts to train the critic $f$ (lines 2--8 in Algorithm~\ref{algo::wgan})
relatively well before each generator update (line 10 in Algorithm \autoref{algo::wgan}),
the loss function at this point is an estimate of the EM distance, up to constant
factors related to the way we constrain the Lipschitz constant of $f$.

Our first experiment illustrates how this estimate correlates well
with the quality of the generated samples. Besides the convolutional
DCGAN architecture, we also ran experiments where we replace the
generator or both the generator and the critic by 4-layer
ReLU-MLP with 512 hidden units.

\begin{figure}[t!]
  \centering
  \begin{minipage}[b]{0.5\linewidth}
    \centering
    \includegraphics[scale=0.3]{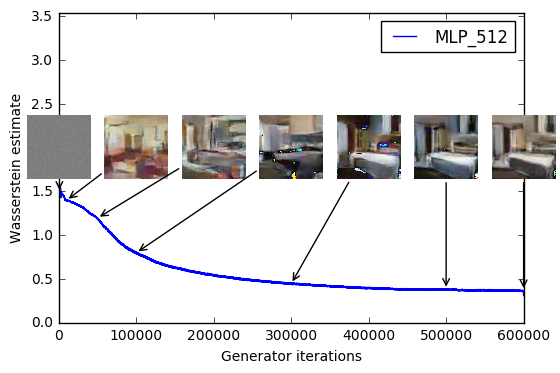}
  \end{minipage}%%
  \begin{minipage}[b]{0.5\linewidth}
    \centering
    \includegraphics[scale=0.3]{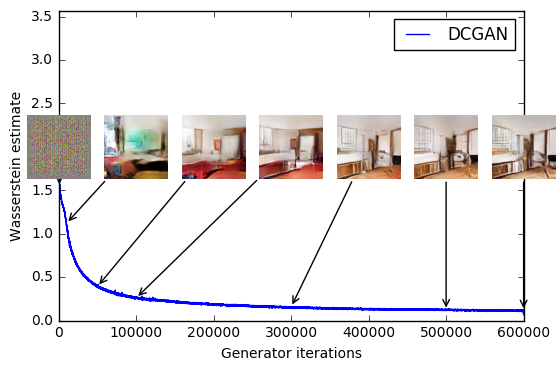}
  \end{minipage}
  \begin{minipage}[b]{0.5\linewidth}
    \centering
    \includegraphics[scale=0.3]{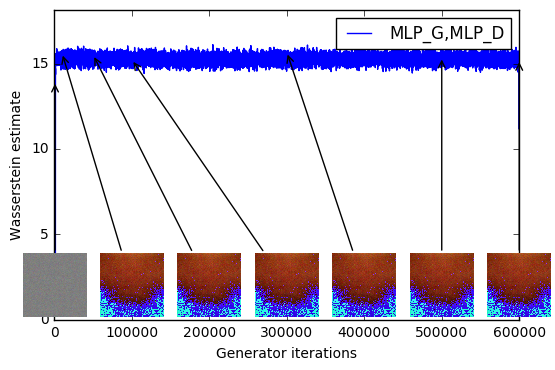}
  \end{minipage}
  \caption{Training curves and samples at different stages
of training. We can see a clear correlation between lower
error and better sample quality.
Upper left: the generator is an MLP with 4 hidden
layers and 512 units at each layer. The loss decreases
constistently as training progresses and sample quality
increases. Upper right: the generator is a standard DCGAN.
The loss decreases quickly and sample quality increases
as well. In both upper plots the critic is a
DCGAN without the sigmoid so losses can be subjected to comparison.
Lower half: both the generator and the discriminator are MLPs
with substantially high learning rates (so training failed). Loss
is constant and samples are constant as well. The training
curves were passed through a median filter for visualization
purposes.}
  \label{fig::corr}
\end{figure}

\begin{figure}[h]
  \centering
  \begin{minipage}[b]{0.5\linewidth}
    \centering
    \includegraphics[scale=0.3]{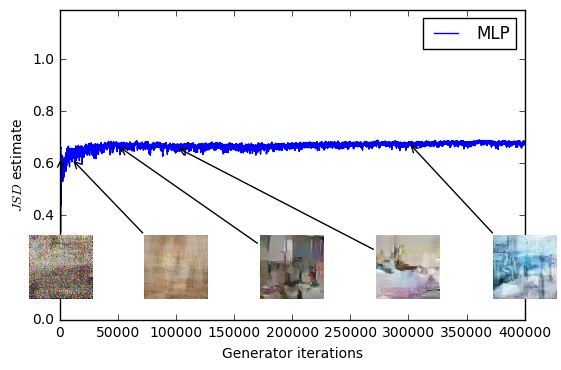}
  \end{minipage}%%
  \begin{minipage}[b]{0.5\linewidth}
    \centering
    \includegraphics[scale=0.3]{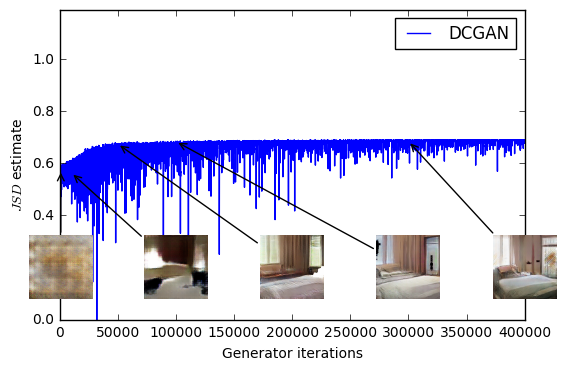}
  \end{minipage}
  \begin{minipage}[b]{0.5\linewidth}
    \centering
    \includegraphics[scale=0.3]{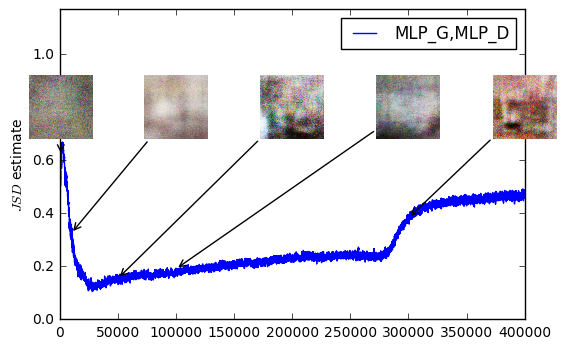}
  \end{minipage}
  \caption{$JS$ estimates for an MLP generator (upper left)
and a DCGAN generator (upper right) trained with the standard
GAN procedure. Both had a DCGAN discriminator. \emph{Both curves
have increasing error}. Samples get better for the DCGAN
but the JS estimate increases or stays constant, pointing towards
no significant correlation between sample quality and loss. Bottom:
$MLP$ with both generator and discriminator. The curve goes up and down
regardless of sample quality.
All training
curves were passed through the same median filter as in \autoref{fig::corr}.}
  \label{fig::gancorr}
\end{figure}

\autoref{fig::corr} plots the evolution of the WGAN estimate
\eqref{eq::wgan} of the EM distance during WGAN training for all three
architectures.  The plots clearly show that these curves correlate
well with the visual quality of the generated samples.

To our knowledge, this is the first time in GAN literature that such a
property is shown, where the loss of the GAN shows properties of
convergence. This property is extremely useful when doing research in
adversarial networks as one does not need to stare at the generated
samples to figure out failure modes and to gain information on which
models are doing better over others.

However, we do not claim that this is a new method to quantitatively
evaluate generative models yet. The constant scaling
factor that depends on the critic's architecture means it's hard
to compare models with different critics. Even more,
in practice the fact that the critic doesn't have
infinite capacity makes it hard to know just how
close to the EM distance our estimate really is.
This being said, we have succesfully used the loss metric to
validate our experiments repeatedly and without failure, and 
we see this as a huge improvement in training GANs which 
previously had no such facility.

In contrast, \autoref{fig::gancorr} plots the evolution of the GAN estimate
of the JS distance during GAN training. More precisely, during GAN training,
the discriminator is trained to maximize
 $$L(D, g_\theta) = \EE_{x \sim \PP_r}[\log D(x)] + \EE_{x \sim \PP_\theta}[\log(1 - D(x))]$$
which is is a lower bound of $2 JS(\PP_r,\PP_\theta) - 2 \log 2$.
In the figure, we plot the quantity $\frac{1}{2} L(D, g_\theta) + \log 2$,
which is a lower bound of the JS distance.

This quantity clearly correlates poorly the sample quality. Note
also that the JS estimate usually stays constant or goes up instead
of going down. In fact
it often remains very close to $\log2\approx0.69$ which is the highest
value taken by the JS distance. In other words, the JS distance
saturates, the discriminator has zero loss, and the generated samples
are in some cases meaningful (DCGAN generator, top right plot) and in
other cases collapse to a single nonsensical image \cite{Good-ea-GAN}.
This last phenomenon has been theoretically explained in
\cite{Arj-ea-Princ} and highlighted in~\cite{Metz-ea-UG}.

When using the $-\log D$ trick \cite{Good-ea-GAN},
the discriminator loss and the generator loss are different.
Figure~\ref{fig::gancorrgencost} in Appendix~\ref{app::gangencost}
reports the same plots for GAN training, but using
the generator loss instead of the discriminator loss.
This does not change the conclusions.

Finally, as a negative result, we report that WGAN training becomes unstable
at times when one uses a momentum based optimizer such as Adam \cite{Kingma-ea-Adam} (with $\beta_1 > 0$)
on the critic, or
when one uses high learning rates. Since the loss for the critic is nonstationary, momentum
based methods seemed to perform worse. We identified
momentum as a potential cause because, as the loss blew up and samples got worse,
the cosine between the Adam step and the gradient usually turned negative. The
only places where this cosine was negative was in these situations of
instability. We therefore switched to RMSProp \cite{Tieleman2012} which is known
to perform well even on very nonstationary problems \cite{Mnih-ea-A3C}.

\subsection{Improved stability}

One of the benefits of WGAN is that it allows us to train the critic
till optimality. When the critic is trained to completion, it simply
provides a loss to the generator that we can train as any other neural
network. This tells us that we no longer need to balance generator and
discriminator's capacity properly. The better the critic, the higher
quality the gradients we use to train the generator.

We observe that WGANs are much more robust than GANs when one varies
the architectural choices for the generator. We illustrate this
by running experiments on three generator architectures:
(1) a convolutional DCGAN generator, (2) a convolutional DCGAN generator
without batch normalization and with a constant number of filters,
and (3) a 4-layer ReLU-MLP with 512 hidden units.
The last two are known to perform very poorly with GANs.
We keep the convolutional DCGAN architecture for
the WGAN critic or the GAN discriminator.

\begin{figure}[t]
    \centering  
    \includegraphics[scale=0.3195]{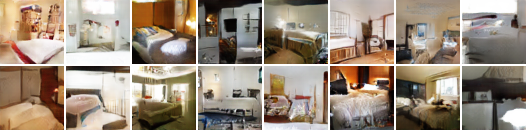}
    \includegraphics[scale=0.327168]{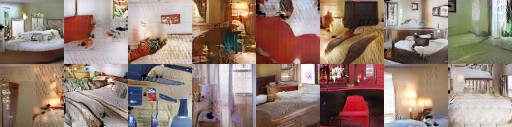}
    \caption{Algorithms trained with a DCGAN generator. Left: WGAN algorithm. Right:
standard GAN formulation. Both algorithms produce high quality samples.}
    \label{fig::dcgan}
%\end{figure}
\par\bigskip    
%\begin{figure}[h]
    \centering  
    \includegraphics[scale=0.28275]{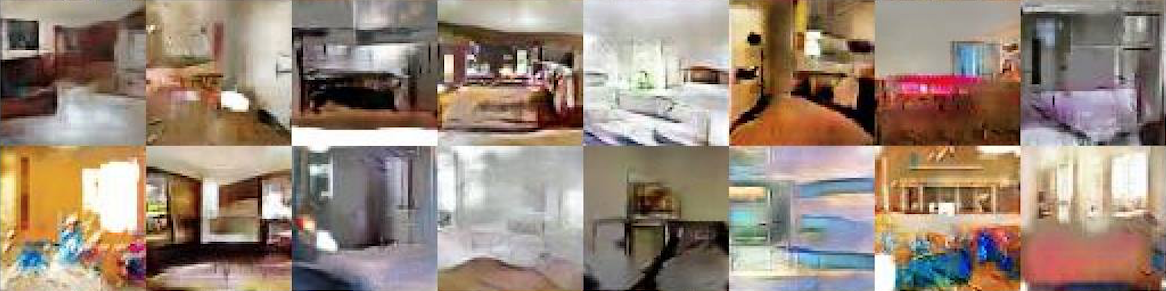}
    \includegraphics[scale=0.3315]{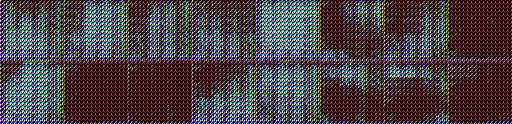}
  \caption{Algorithms trained with a generator without batch normalization
and constant number of filters at every layer
(as opposed to duplicating them every time as in \cite{Rad-ea-DCGAN}). Aside
from taking out batch normalization, the number of parameters is therefore
reduced by a bit more than an order of magnitude. 
Left: WGAN algorithm. Right: standard GAN formulation. As we can see
the standard GAN failed to learn while the WGAN still was able
to produce samples.}
    \label{fig::nobn}
%\end{figure}
\par\bigskip    
%\begin{figure}[h]
    \centering  
    \includegraphics[scale=0.315]{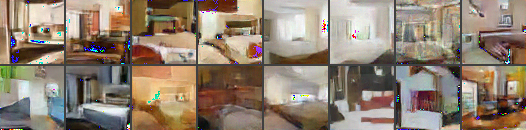}
    \includegraphics[scale=0.322875]{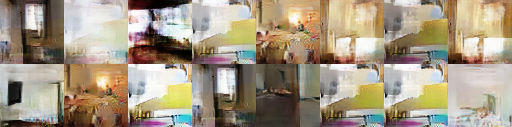}
    \caption{Algorithms trained with an MLP generator with 4 layers and
512 units with ReLU nonlinearities. The number of parameters
is similar to that of a DCGAN, but it lacks a strong inductive bias
for image generation. Left: WGAN algorithm. Right: standard GAN formulation.
The WGAN method still was able to produce samples, lower quality than the DCGAN,
and of higher quality than the MLP of the standard GAN. Note the significant
degree of mode collapse in the GAN MLP.}
    \label{fig::mlp}
\par\bigskip
\end{figure}

Figures~\ref{fig::dcgan}, \ref{fig::nobn}, and \ref{fig::mlp} show
samples generated for these three architectures using both the WGAN
and GAN algorithms. We refer the reader to Appendix
\autoref{app::sheets} for full sheets of generated samples.
Samples were not cherry-picked.

\textbf{In no experiment did we see evidence of mode collapse for the WGAN algorithm.}

%%%%%%%%%%%%%%%%%%%%%%%%%%%%%%%%%%%%%%
\section{Related Work}
\label{sec-related}
There's been a number of works on the so called
Integral Probability Metrics (IPMs) \cite{Muller-IPM}.
Given $\manF$ a set of functions from $\mathcal{X}$
to $\RR$, we can define
\begin{equation} \label{eq::IPM}
  d_\manF(\PP_r, \PP_\theta) = \sup_{f \in \manF} \EE_{x \sim \PP_r}[f(x)] - \EE_{x \sim \PP_\theta}[f(x)]
\end{equation}
as an integral probability metric associated with the
function class $\manF$. It is easily verified that
if for every $f \in \manF$ we have $-f \in \manF$
(such as all examples we'll consider), then
$d_\manF$ is nonnegative, satisfies the triangular
inequality, and is symmetric. Thus, $d_\manF$ is
a pseudometric over $\text{Prob}(\manX)$.

While IPMs might seem to share a similar formula,
as we will see different classes of functions
can yeald to radically different metrics.

\begin{itemize}
\item By the Kantorovich-Rubinstein duality \cite{Vil},
we know that $W(\PP_r, \PP_\theta) = d_\manF(\PP_r,
\PP_\theta)$ when $\manF$ is the set of 1-Lipschitz
functions. Furthermore, if $\manF$ is the set
of $K$-Lipschitz functions, we get $K \cdot W(\PP_r,
\PP_\theta) = d_\manF(\PP_r, \PP_\theta)$.

\item When $\manF$ is the set of all measurable
functions bounded between -1 and 1 (or all
continuous functions between -1 and 1), we
retrieve $d_\manF(\PP_r, \PP_\theta) = \delta(\PP_r,
\PP_\theta)$ the total variation distance \cite{Muller-IPM}.
This already tells us that going from 1-Lipschitz
to 1-Bounded functions drastically changes the
topology of the space, and the regularity
of $d_\manF(\PP_r, \PP_\theta)$ as a loss
function (as by Theorems \autoref{theo::regcost}
and \autoref{theo::dist}).

\item Energy-based GANs (EBGANs) \cite{Zhao-ea-EBGAN}
can be thought of
as the generative approach to the total variation
distance. This connection is stated and
proven in depth in Appendix \autoref{app::ebgan}.
At the core of the connection is that the discriminator
will play the role of $f$ maximizing equation
\eqref{eq::IPM} while its only restriction is
being between $0$ and $m$ for
some constant $m$. This will yeald the same
behaviour as being restricted to be between $-1$
and $1$ up to a constant scaling factor irrelevant
to optimization. Thus, when the discriminator
approaches optimality the cost for the generator
will aproximate the total variation distance
$\delta(\PP_r, \PP_\theta)$.

Since the total variation distance displays the
same regularity as the JS, it can be seen that
EBGANs will suffer from the same problems
of classical GANs regarding not being able
to train the discriminator till optimality
and thus limiting itself to very imperfect
gradients.

\item Maximum Mean Discrepancy (MMD) \cite{Gretton-ea-MMD} is
a specific case of integral probability metrics when
$\manF = \{f \in \mathcal{H}: \|f\|_\infty \leq 1\}$ for
$\mathcal{H}$ some Reproducing Kernel Hilbert Space (RKHS)
associated with a given kernel $k: \manX \times \manX \rightarrow \RR$.
As proved on \cite{Gretton-ea-MMD} we know that MMD is a proper
metric and not only a pseudometric when the kernel is universal.
In the specific case where
$\mathcal{H} = L^2(\manX, m)$ for $m$ the normalized Lebesgue
measure on $\manX$, we know that $\{f \in C_b(\manX), \|f\|_\infty \leq 1\}$
will be contained in $\manF$, and therefore $d_\manF(\PP_r, \PP_\theta)
\leq \delta(\PP_r, \PP_\theta)$ so the regularity of the MMD distance
as a loss function will be at least as bad as the one of the total
variation. Nevertheless this is a very extreme case, since we would
need a very powerful kernel to approximate the whole $L^2$. However,
even Gaussian kernels are able to detect tiny noise patterns
as recently evidenced by \cite{South-ea-MMD}. This points to the
fact that especially with low bandwidth kernels, the distance
might be close to a saturating regime similar as with total
variation or the JS. This obviously doesn't need to
be the case for every kernel, and figuring out how and which different MMDs
are closer to Wasserstein or total variation distances is an interesting
topic of research.

The great aspect of MMD is that via the kernel trick there is no need
to train a separate network to maximize equation \eqref{eq::IPM} for the ball
of a RKHS. However, this has the disadvantage that evaluating the MMD distance
has computational cost that grows quadratically with the amount of samples
used to estimate the expectations in \eqref{eq::IPM}. This last point
makes MMD have limited scalability, and is sometimes inapplicable to
many real life applications because of it. There are estimates with
linear computational cost for the MMD \cite{Gretton-ea-MMD} which
in a lot of cases makes MMD very useful, but they also have worse sample complexity.

\item Generative Moment Matching Networks (GMMNs) \cite{Li-ea-GMMN, Dz-ea-GMMD}
are the generative counterpart of MMD. By backproping through
the kernelized formula for equation \eqref{eq::IPM}, they directly
optimize $d_{MMD}(\PP_r, \PP_\theta)$ (the IPM when $\manF$ is
as in the previous item). As mentioned, this has the advantage
of not requiring a separate network to approximately maximize
equation \eqref{eq::IPM}. However, GMMNs have enjoyed limited applicability.
Partial explanations for their unsuccess are the quadratic cost as a function
of the number of samples and vanishing gradients for low-bandwidth kernels.
Furthermore, it may be possible that some
kernels used in practice are unsuitable for capturing very complex
distances in high dimensional sample spaces such as natural images.
This is properly justified by the fact that \cite{Ramdas-ea-hdk}
shows that for the typical Gaussian MMD test to be reliable (as in it's power
as a statistical test approaching 1), we need the number of
samples to grow linearly with the number of dimensions. Since
the MMD computational cost grows quadratically with the number
of samples in the batch used to estimate equation \eqref{eq::IPM},
this makes the cost of having a reliable estimator
grow quadratically with the number of dimensions, which makes it
very inapplicable for high dimensional problems. Indeed, for
something as standard as $64$x$64$ images, we would need minibatches
of size at least $4096$ (without taking into account the constants
in the bounds of \cite{Ramdas-ea-hdk} which would make this number
substantially larger) and a total cost per iteration of
$4096^2$, over 5 orders of magnitude more than a
GAN iteration when using the standard batch size of 64.

That being said, these numbers can be a bit unfair to the MMD,
in the sense that we are comparing empirical sample complexity of
GANs with the theoretical sample complexity of MMDs, which tends
to be worse. However, in the original GMMN paper \cite{Li-ea-GMMN} they
indeed used a minibatch of size 1000, much larger than the standard
32 or 64 (even when this incurred in quadratic computational cost).
While estimates that have linear computational cost as a function of the number
of samples exist \cite{Gretton-ea-MMD},
they have worse sample complexity, and to the best of
our knowledge they haven't been yet applied in a generative context
such as in GMMNs.

\end{itemize}

On another great line of research, the recent work of \cite{Mont-ea-WRBM}
has explored the use of Wasserstein distances in the context of learning
for Restricted Boltzmann Machines for discrete spaces. The motivations
at a first glance might seem quite different, since the manifold setting is restricted
to continuous spaces and in finite discrete spaces the weak and strong
topologies (the ones of W and JS respectively) coincide. However, in the
end there is more in commmon than not about our motivations. We both
want to compare distributions in a way that leverages the geometry of the
underlying space, and Wasserstein allows us to do exactly that.

Finally, the work of \cite{Gen-ea-LSOpt} shows new algorithms for calculating
Wasserstein distances between different distributions. We believe this direction
is quite important, and perhaps could lead to new ways of evaluating generative models.

\section{Conclusion}
  We introduced an algorithm that we deemed WGAN, an alternative
to traditional GAN training. In this new model, we showed that we can
improve the stability of learning, get rid of problems like mode collapse,
and provide meaningful learning curves useful for debugging and hyperparameter
searches. Furthermore, we showed that the corresponding optimization problem
is sound, and provided extensive theoretical work highlighting the
deep connections to other distances between distributions.

\section*{Acknowledgments}
We would like to thank
Mohamed Ishmael Belghazi,
Emily Denton,
Ian Goodfellow,
Ishaan Gulrajani,
Alex Lamb,
David Lopez-Paz,
Eric Martin,
Maxime Oquab,
Aditya Ramesh,
Ronan Riochet,
Uri Shalit,
Pablo Sprechmann,
Arthur Szlam,
Ruohan Wang,
for helpful comments and advice.

\bibliography{arxiv}
\bibliographystyle{plain}

\clearpage

\begin{appendices}
\section{Why Wasserstein is indeed weak} \label{app::weak}
We now introduce our notation. Let $\mathcal{X} \subseteq{\RR^d}$
be a compact set (such as $[0,1]^d$ the space of images). We define
$\text{Prob}(\manX)$ to be the space of probability measures over
$\manX$. We note
$$C_b(\manX) = \{f: \manX \to \RR \text{, $f$ is continuous and bounded}\}$$
Note that if $f \in C_b(\manX)$,
we can define $\|f\|_\infty = \max_{x \in \manX} |f(x)|$, since
$f$ is bounded. With this norm, the space $(C_b(\manX), \|\cdot\|_\infty)$
is a normed vector space. As for any normed vector space, we can define
its dual
$$ C_b(\manX)^* = \{\phi: C_b(\manX) \rightarrow \RR \text{, $\phi$ is linear
and continuous} \} $$
and give it the dual norm $\|\phi\| = \sup_{f \in C_b(\manX), \|f\|_\infty \leq 1} |\phi(f)|$.

With this definitions, $(C_b(\manX)^*, \|\cdot\|)$ is another normed space.
Now let $\mu$ be a signed measure over $\manX$, and let us define
the total variation distance
$$\|\mu\|_{TV} = \sup_{A \subseteq \manX} |\mu(A)|$$
where the supremum is taken all Borel sets in $\manX$.
Since the total variation is a norm, then if we have $\PP_r$
and $\PP_\theta$ two probability distributions over $\manX$,
$$\delta(\PP_r, \PP_\theta) := \|\PP_r - \PP_\theta\|_{TV}$$
is a distance in $\text{Prob}(\manX)$ (called the total variation
distance).

We can consider
$$ \Phi: (\text{Prob}(\manX), \delta) \rightarrow (C_b(\manX)^*, \|\cdot \|)$$
where $\Phi(\PP)(f) := \EE_{x \sim \PP}[f(x)]$ is a linear function
over $C_b(\manX)$. The Riesz Representation theorem (\cite{Kakutani-Riesz},
Theorem 10) tells us that $\Phi$ is an isometric immersion. This
tells us that we can effectively consider $\text{Prob}(\manX)$ with
the total variation distance as a subset of $C_b(\manX)^*$ with
the norm distance. Thus, just to accentuate it one more time,
the total variation over $\text{Prob}(\manX)$ is exactly
the norm distance over $C_b(\manX)^*$.

Let us stop for a second and analyze what all this technicality meant.
The main thing to carry is that we introduced a distance $\delta$
over probability distributions. When looked as a distance over
a subset of $C_b(\manX)^*$, this distance gives the norm topology.
The norm topology is very strong. Therefore, we can expect that
not many functions $\theta \mapsto \PP_\theta$ will be continuous
when measuring distances between distributions with $\delta$. As
we will show later in Theorem \autoref{theo::dist}, $\delta$ gives the same topology
as the Jensen-Shannon divergence, pointing to the fact that the
JS is a very strong distance, and is thus more propense to
give a discontinuous loss function.

Now, all dual spaces (such as $C_b(\manX)^*$ and thus
$\text{Prob}(\manX)$) have a strong topology (induced by the norm),
and a weak* topology. As the name suggests, the weak* topology
is much weaker than the strong topology. In the case of
$\text{Prob}(\manX)$, the strong topology is given by the
total variation distance, and the weak* topology is given
by the Wasserstein distance (among others) \cite{Vil}.
\section{Assumption definitions}

\begin{ass} \label{ass::lip}
Let $g: \mathcal{Z} \times \RR^d \rightarrow \mathcal{X}$ be
locally Lipschitz between finite dimensional vector spaces.
We will denote $g_\theta(z)$ it's evaluation on coordinates
$(z, \theta)$. We say that $g$ satisfies assumption \ref{ass::lip}
for a certain probability distribution $p$ over $\mathcal{Z}$
if there are local Lipschitz constants $L(\theta, z)$ such
that
$$ \EE_{z \sim p}[L(\theta, z)] < +\infty $$ 
\end{ass}

\section{Proofs of things} \label{app::proofs}

\begin{proof}[Proof of Theorem \ref{theo::regcost}]
Let $\theta$ and $\theta'$ be two parameter vectors in $\RR^d$. Then, we 
will first attempt to bound $W(\PP_\theta, \PP_{\theta'})$, from where the
theorem will come easily. The main element of the proof is the use of the
coupling $\gamma$, the distribution of the joint $(g_\theta(Z), g_{\theta'}(Z))$,
which clearly has $\gamma \in \Pi(\PP_\theta, \PP_{\theta'})$.

By the definition of the Wasserstein distance, we have
\begin{align*}
  W(\PP_\theta, \PP_{\theta'}) &\leq \int_{\manX \times \manX} \|x - y\| \diff \gamma \\
&= \EE_{(x, y) \sim \gamma} [\|x - y \|] \\
&= \EE_{z}[\|g_\theta(z) - g_{\theta'}(z)\|]
\end{align*}
If $g$ is continuous in $\theta$, then $g_\theta(z) \to_{\theta \to \theta'} g_{\theta'}(z)$,
so $\|g_\theta - g_{\theta'}\| \to 0$ pointwise as functions of $z$. Since
$\manX$ is compact, the distance of any two elements in it has to be uniformly bounded by
some constant $M$, and therefore $\|g_\theta(z) - g_{\theta'}(z)\| \leq M$ for
all $\theta$ and $z$ uniformly. By the bounded convergence theorem, we therefore
have
$$ W(\PP_\theta, \PP_{\theta'}) \leq \EE_{z}[\|g_\theta(z) - g_{\theta'}(z)\|] \to_{\theta \to \theta'} 0 $$
Finally, we have that
$$ |W(\PP_r, \PP_\theta) - W(\PP_r, \PP_{\theta'}) | \leq W(\PP_\theta, \PP_{\theta'}) \to_{\theta \to \theta'} 0 $$
proving the continuity of $W(\PP_r, \PP_\theta)$.

Now let $g$ be locally Lipschitz. Then, for a given
pair $(\theta, z)$ there is a constant $L(\theta, z)$
and an open set $U$ such that $(\theta, z) \in U$,
such that for every $(\theta', z') \in U$ we have
$$ \|g_\theta(z) - g_\theta'(z')\| \leq L(\theta, z) (\|\theta - \theta'\| + \|z - z'\|) $$

By taking expectations and $z'=z$ we
$$ \EE_z[\|g_\theta(z) - g_{\theta'}(z)\|] 
  \leq \|\theta - \theta'\| \EE_z[L(\theta, z)] $$
whenever $(\theta', z) \in U$. Therefore, we can define
$U_\theta = \{\theta' | (\theta', z) \in U\}$. It's easy
to see that since $U$ was open, $U_\theta$ is as well.
Furthermore, by assumption \ref{ass::lip}, we can
define $L(\theta) = \EE_z[L(\theta, z)]$ and achieve
$$ |W(\PP_r, \PP_\theta) - W(\PP_r, \PP_{\theta'})| 
\leq W(\PP_\theta, \PP_{\theta'}) 
\leq L(\theta) \|\theta - \theta'\| $$
for all $\theta' \in U_\theta$, meaning that $W(\PP_r, \PP_\theta)$
is locally Lipschitz. This obviously implies that
$W(\PP_r, \PP_\theta)$ is everywhere continuous, and
by Radamacher's theorem we know it has to be differentiable
almost everywhere.

The counterexample for item 3 of the Theorem is indeed
Example \autoref{exa::lines}.
\end{proof}

\begin{proof}[Proof of Corollary \autoref{coro::nnregcost}]
We begin with the case of smooth nonlinearities. Since $g$ is
$C^1$ as a function of $(\theta, z)$ then for any
fixed $(\theta, z)$ we have $L(\theta, Z) \leq \|\nabla_{\theta, x}g_\theta(z)\| + \epsilon$
is an acceptable local Lipschitz constant for all $\epsilon > 0$.
Therefore, it suffices to prove
$$\EE_{z \sim p(z)}[\|\nabla_{\theta, z} g_\theta(z)\|]< +\infty$$
If $H$ is the number of layers we know
that $\nabla_z g_\theta(z) = \prod_{k=1}^H W_k D_k$ where
$W_k$ are the weight matrices and $D_k$ is are the diagonal Jacobians
of the nonlinearities. Let $f_{i:j}$ be the application
of layers $i$ to $j$ inclusively (e.g. $g_\theta = f_{1:H}$).
Then, $\nabla_{W_k} g_\theta(z) = \left(\left(\prod_{i=k+1}^H W_i D_i
\right) D_k \right) f_{1:k-1}(z)$.
We recall that if $L$ is the Lipschitz constant
of the nonlinearity, then $\|D_i\|\leq L$ and
$\|f_{1:k-1}(z)\| \leq \|z\| L^{k-1} \prod_{i=1}^{k-1}W_i$. Putting this together,
\begin{align*}
  \|\nabla_{z, \theta} g_\theta(z)\| &\leq \|\prod_{i=1}^H W_i D_i\|
+ \sum_{k=1}^H \|\left(\left(\prod_{i=k+1}^H W_i D_i \right) D_k \right) 
f_{1:k-1}(z)\| \\
  &\leq L^H \prod_{i=H}^K \|W_i\| + \sum_{k=1}^H \|z\| L^H 
\left(\prod_{i=1}^{k-1} \|W_i\| \right)
\left(\prod_{i=k+1}^{H} \|W_i\| \right)
\end{align*}
If $C_1(\theta) = L^H\left(\prod_{i=1}^H\|W_i\|\right)$ and
$C_2(\theta) = \sum_{k=1}^H L^H 
\left(\prod_{i=1}^{k-1} \|W_i\| \right)
\left(\prod_{i=k+1}^{H} \|W_i\| \right)$ then
$$\EE_{z \sim p(z)}[\|\nabla_{\theta, z} g_\theta(z)\|] 
\leq C_1(\theta) + C_2(\theta) \EE_{z \sim p(z)}[\|z\|] < +\infty$$
finishing the proof
\end{proof}

\begin{proof}[Proof of Theorem \ref{theo::dist}] \mbox{~}\par
  \begin{enumerate}
    \item
      \begin{itemize}
      \item ($\delta(\PP_n, \PP) \to 0 \Rightarrow JS(\PP_n,\PP) \to 0$) \quad --- \quad
        Let $\PP_m$ be the mixture distribution $\PP_m = \frac{1}{2} \PP_n
        + \frac{1}{2} \PP$ (note that $\PP_m$ depends on $n$).
        It is easily verified that $\delta(\PP_m, \PP_n)
        \leq \delta(\PP_n, \PP)$, and in particular this tends to 0 (as
        does $\delta(\PP_m, \PP)$). We now show this for completeness.
        Let $\mu$ be a signed measure,
        we define $\|\mu\|_{TV} = \sup_{A \subseteq \mathcal{X}} |\mu(A)|$.
        for all Borel sets $A$.
        In this case,
        \begin{align*}
          \delta(\PP_m, \PP_n) &= \| \PP_m - \PP_n \|_{TV} \\
          &= \| \frac{1}{2} \PP + \frac{1}{2} \PP_n - \PP_n \|_{TV} \\
          &= \frac{1}{2} \| \PP - \PP_n \|_{TV} \\
          &= \frac{1}{2} \delta(\PP_n, \PP) \leq \delta(\PP_n, \PP)
        \end{align*}
        
        Let $f_n = \frac{d \PP_n}{d \PP_m}$ be the Radon-Nykodim
        derivative between $\PP_n$ and the mixture. Note that by
        construction for every Borel set $A$ we have $\PP_n(A) \leq 
        2 \PP_m(A)$. If $A = \{f_n > 3\}$ then we get
        $$ \PP_n(A) = \int_A f_n \d \PP_m \geq 3 \PP_m(A)$$
        which implies $\PP_m(A) = 0$. This means that $f_n$
        is bounded by 3 $\PP_m$(and therefore $\PP_n$ and
        $\PP$)-almost everywhere. We could have done this
        for any constant larger than 2 but for our
        purposes 3 will sufice.
        
        Let $\epsilon > 0$ fixed,
        and $A_n = \{f_n > 1 + \epsilon\}$. Then, 
        $$ \PP_n(A_n) = \int_{A_n} f_n \d \PP_m \geq (1 + \epsilon) \PP_m(A_n)$$
        Therefore,
        \begin{align*}
          \epsilon \PP_m(A_n) &\leq \PP_n(A_n) - \PP_m(A_n) \\
          &\leq |\PP_n(A_n) - \PP_m(A_n)| \\
          &\leq \delta(\PP_n, \PP_m) \\
          &\leq \delta(\PP_n, \PP).
        \end{align*}
        Which implies $\PP_m(A_m) \leq \frac{1}{\epsilon} \delta(\PP_n, \PP)$. Furthermore,
        \begin{align*}
          \PP_n(A_n) &\leq \PP_m(A_n) + |\PP_n(A_n) - \PP_m(A_n)| \\
          &\leq \frac{1}{\epsilon} \delta(\PP_n, \PP) + \delta(\PP_n, \PP_m) \\
          &\leq \frac{1}{\epsilon} \delta(\PP_n, \PP) + \delta(\PP_n, \PP) \\
          &\leq \left(\frac{1}{\epsilon} + 1\right) \delta(\PP_n, \PP)
        \end{align*}
        We now can see that
        \begin{align*}
          KL(\PP_n \| \PP_m) &= \int \log(f_n) \d \PP_n \\
          &\leq \log(1 + \epsilon) + \int_{A_n} \log(f_n) \d \PP_n \\
          &\leq \log(1 + \epsilon) + \log(3) \PP_n(A_n) \\
          &\leq \log(1 + \epsilon) + \log(3) \left(\frac{1}{\epsilon} + 1\right) \delta(\PP_n, \PP)
        \end{align*}
        Taking limsup we get $0 \leq \limsup KL(\PP_n \| \PP_m) \leq \log(1 + \epsilon)$
        for all $\epsilon > 0$, which means $KL(\PP_n \| \PP_m) \to 0$.
        
        In the same way, we can define $g_n = \frac{d \PP}{d \PP_m}$, and
        $$2 \PP_m(\{g_n > 3\}) \geq \PP(\{g_n > 3\}) \geq 3 \PP_m(\{g_n > 3\}) $$
        meaning that $\PP_m(\{g_n > 3\}) = 0$ and therefore
        $g_n$ is bounded by 3 almost everywhere for $\PP_n, \PP_m$
        and $\PP$. With the same calculation, $B_n = \{g_n > 1 + \epsilon\}$ and
        $$\PP(B_n) = \int_{B_n} g_n \d \PP_m \geq (1 + \epsilon) \PP_m(B_n) $$
        so $\PP_m(B_n) \leq \frac{1}{\epsilon} \delta(\PP, \PP_m) \to 0$, and therefore
        $\PP(B_n) \to 0$. We can now show
        \begin{align*}
          KL(\PP \| \PP_m) &= \int \log(g_n) \d \PP \\
          &\leq \log(1 + \epsilon) + \int_{B_n} \log (g_n) \d \PP \\
          &\leq \log(1 + \epsilon) + \log(3) \PP(B_n)
        \end{align*}
        so we achieve $ 0 \leq \limsup KL(\PP \| \PP_m) \leq \log(1 + \epsilon)$
        and then $KL(\PP \| \PP_m) \to 0$. Finally, we conclude
        $$JS(\PP_n,\PP) = \frac{1}{2} KL(\PP_n \| \PP_m) + \frac{1}{2} KL(\PP \| \PP_m) \to 0 $$
        
      \item ($JS(\PP_n,\PP) \to 0 \Rightarrow \delta(\PP_n, \PP) \to 0$) \quad --- \quad
        by a simple
        application of the triangular and Pinsker's inequalities we get
        \begin{align*}
          \delta(\PP_n, \PP) &\leq \delta(\PP_n, \PP_m) + \delta(\PP, \PP_m) \\
          &\leq \sqrt{\frac{1}{2} KL(\PP_n \| \PP_m)} + \sqrt{\frac{1}{2} KL(\PP \| \PP_m)} \\
          &\leq 2 \sqrt{JS(\PP_n,\PP)} \to 0
        \end{align*}
      \end{itemize}
    \item
      This is a long known fact that $W$ metrizes
      the weak* topology of $(C(\manX), \|\cdot \|_\infty)$
      on $\text{Prob}(\manX)$, and by definition this
      is the topology of convergence in distribution.
      A proof of this can be found (for example) in \cite{Vil}.
    \item
      This is a straightforward application of Pinsker's inequality
      \begin{align*}
        \delta(\PP_n, \PP) \leq \sqrt{\frac{1}{2} KL(\PP_n \| \PP)} \to 0 \\
        \delta(\PP, \PP_n) \leq \sqrt{\frac{1}{2} KL(\PP \| \PP_n)} \to 0 \\
      \end{align*}
      \vspace{-8mm}
    \item
      This is trivial by recalling the fact that $\delta$ and $W$ give the
      strong and weak* topologies on the dual of $(C(\manX), \|\cdot \|_\infty)$
      when restricted to $\text{Prob}(\manX)$.
  \end{enumerate}
\end{proof}

\begin{proof}[Proof of Theorem \autoref{theo::gradW}]
  Let us define
  \begin{align*}
    V(\tilde f, \theta) &= \EE_{x \sim \PP_r}[\tilde f(x)] - \EE_{x \sim \PP_\theta} [\tilde f(x)] \\
    &= \EE_{x \sim \PP_r}[\tilde f(x)] - \EE_{z \sim p(z)} [\tilde f(g_\theta(z))]
  \end{align*}
  where $\tilde f$ lies in $\manF = \{\tilde f: \manX \to \RR \text{ , $\tilde f \in C_b(\manX)$, $\|\tilde f\|_L \leq 1$}\}$ and
  $\theta \in \RR^d$.
  
  Since $\manX$ is compact, we know 
  by the Kantorovich-Rubenstein duality \cite{Vil} that
  there is an $f \in \manF$ that attains the value
  \begin{equation*}
    W(\PP_r, \PP_\theta) = \sup_{\tilde f \in \manF} V(\tilde f, \theta) = V(f, \theta)
  \end{equation*}
  Let us define $X^*(\theta) = \{f \in \manF: V(f, \theta) = W(\PP_r, \PP_\theta)\}$. By
  the above point we know then that $X^*(\theta)$ is non-empty. We know
  that by a simple envelope theorem (\cite{Milgrom-ea-envelope}, Theorem 1) that
  $$ \nabla_\theta W(\PP_r, \PP_\theta) = \nabla_\theta V(f, \theta) $$
  for any $f \in X^*(\theta)$ when both terms are well-defined.
  
  Let $f \in X^*(\theta)$, which we knows exists
  since $X^*(\theta)$ is non-empty for all $\theta$. Then, we get 
  \begin{align*}
    \nabla_\theta W(\PP_r, \PP_\theta) &= \nabla_\theta V(f, \theta) \\
    &= \nabla_\theta[ \EE_{x \sim \PP_r}[f(x)] - \EE_{z \sim p(z)}[f(g_\theta(z))] \\
      &= -\nabla_\theta \EE_{z \sim p(z)}[f(g_\theta(z))]
  \end{align*}
  under the condition that the first and last terms are well-defined.
  The rest of the proof will be dedicated to show that
  \begin{equation} \label{eq::difsign}
    -\nabla_\theta \EE_{z \sim p(z)}[f(g_\theta(z))] = -\EE_{z \sim p(z)}[\nabla_\theta f(g_\theta(z))]
  \end{equation}
  when the right hand side is defined. For the reader who is
  not interested in such technicalities, he or she can skip the
  rest of the proof.
  
  Since $f \in \manF$, we know that it is 1-Lipschitz.
  Furthermore, $g_\theta(z)$ is locally
  Lipschitz as a function of $(\theta, z)$. Therefore, 
  $f(g_\theta(z))$ is locally Lipschitz on $(\theta, z)$
  with constants $L(\theta, z)$ (the same ones as $g$).
  By Radamacher's Theorem, $f(g_\theta(z))$ has to be
  differentiable almost everywhere for $(\theta, z)$
  jointly. Rewriting this, the set $A = \{(\theta, z):
  \text{$f \circ g$ is not differentiable}\}$ has
  measure 0. By Fubini's Theorem, this implies that
  for almost every $\theta$ the section $A_\theta
  = \{z: (\theta, z) \in A\}$ has measure 0.
  Let's now fix a $\theta_0$ such that 
  the measure of $A_{\theta_0}$ is null (\textbf{such
    as when the right hand side of equation \eqref{eq::difsign}
    is well defined}). For this
  $\theta_0$ we have $\nabla_\theta f(g_\theta(z))|_{\theta_0}$
  is well-defined for almost any $z$, and since $p(z)$
  has a density, it is defined $p(z)$-a.e. By assumption
  \ref{ass::lip} we know that
  $$ \EE_{z \sim p(z)} [\|\nabla_\theta f(g_\theta(z))|_{\theta_0}\|]
  \leq \EE_{z \sim p(z)} [L(\theta_0, z)] < + \infty $$
  so $\EE_{z \sim p(z)} [\nabla_\theta f(g_\theta(z))|_{\theta_0}]$
  is well-defined for almost every $\theta_0$. Now, we can see
  \begin{equation}
    \frac{\EE_{z\sim p(z)}[f(g_\theta(z))] - \EE_{z \sim p(z)}[f(g_{\theta_0}(z))]
      - \langle (\theta - \theta_0), \EE_{z \sim p(z)} [\nabla_\theta f(g_\theta(z))|_{\theta_0}]
      \rangle}{\|\theta - \theta_0\|} \label{eq::diffquot}
  \end{equation}
  \begin{align*}
    &= \EE_{z \sim p(z)}\left[ \frac{f(g_\theta(z)) - f(g_{\theta_0}(z))
        - \langle(\theta - \theta_0), \nabla_\theta f(g_\theta(z))|_{\theta_0}
        \rangle}{\|\theta - \theta_0\|}\right] \\
  \end{align*}
  By differentiability, the term inside the integral converges $p(z)$-a.e. to 0
  as $\theta \to \theta_0$. Furthermore,
  \begin{multline*}
    \|\frac{f(g_\theta(z)) - f(g_{\theta_0}(z))
      - \langle(\theta - \theta_0), \nabla_\theta f(g_\theta(z))|_{\theta_0}
      \rangle}{\|\theta - \theta_0\|}\|
    \\ \leq \frac{\|\theta - \theta_0\| L(\theta_0, z)
      + \|\theta - \theta_0\| \| \nabla_\theta f(g_\theta(z))|_{\theta_0}
      \|}{\|\theta - \theta_0\|}\\
    \leq 2 L(\theta_0, z)
  \end{multline*}
  and since $\EE_{z \sim p(z)}[2 L(\theta_0, z)] < +\infty$ by assumption 1,
  we get by dominated convergence that Equation \ref{eq::diffquot} converges
  to 0 as $\theta \to \theta_0$ so
  $$ \nabla_\theta
  \EE_{z \sim p(z)} [f(g_\theta(z))] = \EE_{z \sim p(z)} [\nabla_\theta f(g_\theta(z))]
  $$
  for almost every $\theta$, and in particular when the right
  hand side is well defined.
  Note that the mere existance of the left
  hand side (meaning the differentiability a.e. of $\EE_{z \sim p(z)}
  [f(g_\theta(z))]$) had to be proven, which we just did.
\end{proof}

\clearpage
\section{Energy-based GANs optimize total variation} \label{app::ebgan}
In this appendix we show that under an optimal discriminator,
energy-based GANs (EBGANs) \cite{Zhao-ea-EBGAN} optimize the total variation
distance between the real and generated distributions.

Energy-based GANs are trained in a similar fashion to GANs, only under
a different loss function. They have a discriminator $D$ who tries to
minimize 
\begin{equation*}
L_D(D, g_\theta) = \EE_{x \sim \PP_r}[D(x)] + \EE_{z \sim p(z)}[[m - D(g_\theta(z))]^+]
\end{equation*}
for some $m > 0$ and $[x]^+ = \max(0, x)$ and a generator network $g_\theta$ that's trained to minimize
\begin{equation*}
L_G(D, g_\theta) = \EE_{z \sim p(z)} [D(g_\theta(z))] - \EE_{x \sim \PP_r}[D(x)]
\end{equation*}
Very importantly, $D$ is constrained to be non-negative,
since otherwise the trivial solution for $D$ would be to set everything to
arbitrarily low values. The original EBGAN paper used only $\EE_{z \sim p(z)}[D(g_\theta(z))]$ for
the loss of the generator, but this is obviously equivalent to our
definition since the term $\EE_{x \sim \PP_r}[D(x)]$ does not dependent
on $\theta$ for a fixed discriminator (such as when backproping to the
generator in EBGAN
training) and thus minimizing one or the other is equivalent.

We say that a measurable function $D^*: \mathcal{X} \rightarrow [0, +\infty)$
is optimal for $g_\theta$ (or $\PP_\theta)$ if $L_D(D^*, g_\theta) \leq L_D(D, g_\theta)$ for all
other measurable functions $D$. We show that such a discriminator
always exists for any two distributions $\PP_r$ and $\PP_\theta$,
and that under such a discriminator, $L_G(D^*, g_\theta)$ is
proportional to $\delta(\PP_r, \PP_\theta)$. As a simple corollary,
we get the fact that $L_G(D^*, g_\theta)$ attains its minimum
value if and only if $\delta(\PP_r, \PP_\theta)$ is at its
minimum value, which is 0, and $\PP_r = \PP_\theta$
(Theorems 1-2 of \cite{Zhao-ea-EBGAN}).

\begin{theo} Let $\PP_r$ be a the real data distribution over
a compact space $\mathcal{X}$.
Let $g_\theta: \mathcal{Z} \rightarrow \mathcal{X}$ be a
measurable function (such as any neural network). Then, 
an optimal discriminator $D^*$ exists for $\PP_r$ and
$\PP_\theta$, and 
$$ L_G(D^*, g_\theta) = \frac{m}{2} \delta(\PP_r, \PP_\theta) $$
\end{theo}
\begin{proof} First, we prove that there exists an optimal
discriminator. Let $D: \mathcal{X} \rightarrow [0, +\infty)$
be a measurable function, then $D'(x) := \min(D(x), m)$ is
also a measurable function, and $L_D(D', g_\theta) \leq L_D(D, g_\theta)$.
Therefore, a function $D^*: \mathcal{X} \rightarrow [0, +\infty)$ is
optimal if and only if ${D^*}'$ is. Furthermore, it is optimal if and
only if $L_D(D^*, g_\theta) \leq L_D(D, g_\theta)$ for all $D: \mathcal{X}
\rightarrow [0, m]$. We are then interested to see if there's an
optimal discriminator for the problem
$\min_{0 \leq D(x) \leq m} L_D(D, g_\theta)$.

Note now that if $0 \leq D(x) \leq m$ we have
\begin{align*}
  L_D(D, g_\theta) &= \EE_{x \sim \PP_r}[D(x)] + \EE_{z \sim p(z)}[[m - D(g_\theta(z))]^+] \\
    &= \EE_{x \sim \PP_r}[D(x)] + \EE_{z \sim p(z)}[m - D(g_\theta(z))] \\
    &= m + \EE_{x \sim \PP_r}[D(x)] - \EE_{z \sim p(z)}[D(g_\theta(z))] \\
    &= m + \EE_{x \sim \PP_r}[D(x)] - \EE_{x \sim \PP_\theta}[D(x)]
\end{align*}
Therefore, we know that
\begin{align*}
  \inf_{0 \leq D(x) \leq m} L_D(D, g_\theta) &= m + \inf_{0 \leq D(x) \leq m} \EE_{x \sim \PP_r}[D(x)] - \EE_{x \sim \PP_\theta}[D(x)] \\
    &= m + \inf_{-\frac{m}{2} \leq D(x) \leq \frac{m}{2}} \EE_{x \sim \PP_r}[D(x)] - \EE_{x \sim \PP_\theta}[D(x)] \\
    &= m + \frac{m}{2} \inf_{-1 \leq f(x) \leq 1} \EE_{x \sim \PP_r}[f(x)] - \EE_{x \sim \PP_\theta}[f(x)]
\end{align*}
The interesting part is that
\begin{equation} \label{eq::radon}
\inf_{-1 \leq f(x) \leq 1} \EE_{x \sim \PP_r}[f(x)] - \EE_{x \sim \PP_\theta}[f(x)] = - \delta(\PP_r, \PP_\theta)
\end{equation}
and there is an $f^*: \manX \rightarrow [-1,1]$ such that $\EE_{x \sim \PP_r}[f^*(x)] -
\EE_{x \sim \PP_\theta}[f^*(x)] = - \delta(\PP_r, \PP_\theta)$. This is a long known fact,
found for example in \cite{Vil}, but we prove it later for completeness. In that case,
we define $D^*(x) = \frac{m}{2}f^*(x) + \frac{m}{2}$. We then have $0 \leq D(x) \leq m$ and
\begin{align*}
  L_D(D^*, g_\theta) 
    &= m + \EE_{x \sim \PP_r}[D^*(x)] - \EE_{x \sim \PP_\theta}[D^*(x)] \\
    &= m + \frac{m}{2} \EE_{x \sim \PP_r}[D^*(x)] - \EE_{x \sim \PP_\theta}[f^*(x)] \\
    &= m - \frac{m}{2} \delta(\PP_r, \PP_\theta) \\
    &= \inf_{0 \leq D(x) \leq m} L_D(D, g_\theta)
\end{align*}
This shows that $D^*$ is optimal and $L_D(D^*, g_\theta) = m - \frac{m}{2} \delta(\PP_r, \PP_\theta)$. Furthermore,
\begin{align*}
L_G(D^*, g_\theta) &= \EE_{z \sim p(z)}[D^*(g_\theta(z))] - \EE_{x \sim \PP_r}[D^*(x)] \\
  &= -L_D(D^*, g_\theta) + m \\
  &= \frac{m}{2} \delta(\PP_r, \PP_g)
\end{align*}
concluding the proof.

For completeness, we now show a proof for equation \eqref{eq::radon}
and the existence of said $f^*$ that attains the value of the infimum.
Take $\mu = \PP_r - \PP_\theta$,
which is a signed measure, and $(P, Q)$ its Hahn decomposition.
Then, we can define $f^* := \mathbbm{1}_Q - \mathbbm{1}_P$.
By construction, then
\begin{align*}
EE_{x \sim \PP_r}[f^*(x)] - \EE_{x \sim \PP_\theta}[f^*(x)]
&= \int f^* \d \mu = \mu(Q) - \mu(P) \\
&= -(\mu(P) - \mu(Q)) = -\|\mu\|_{TV} \\
&= -\|\PP_r - \PP_\theta\|_{TV} \\
&= -\delta(\PP_r, \PP_\theta)
\end{align*}
Furthermore, if $f$ is bounded between -1 and 1, we get
\begin{align*}
|\EE_{x \sim \PP_r}[f(x)] - \EE_{x \sim \PP_\theta}[f(x)] |
  &= |\int f \d \PP_r - \int f \d \PP_\theta| \\
  &= |\int f \d \mu| \\
  &\leq \int |f| \d |\mu| \leq \int 1 \d |\mu|\\
  &=|\mu|(\mathcal{X}) = \|\mu\|_{TV} = \delta(\PP_r, \PP_\theta)
\end{align*}
Since $\delta$ is positive, we can conclude $\EE_{x \sim \PP_r}[f(x)] - \EE_{x \sim \PP_\theta}[f(x)] \geq -\delta(\PP_r, \PP_\theta)$.
\end{proof}
\clearpage
\pagebreak

\section{Generator's cost during normal GAN training} \label{app::gangencost}
\begin{figure}[h]
  \centering
  \begin{minipage}[b]{0.5\linewidth}
    \centering
    \includegraphics[scale=0.3]{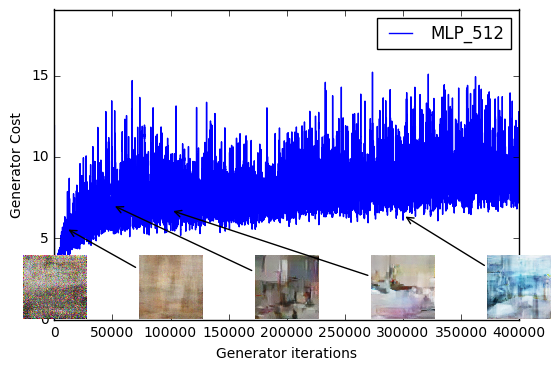}
  \end{minipage}%%
  \begin{minipage}[b]{0.5\linewidth}
    \centering
    \includegraphics[scale=0.3]{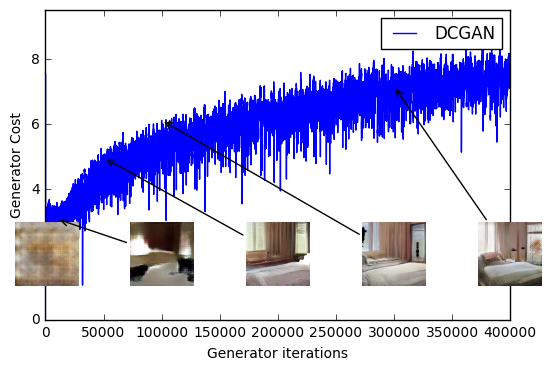}
  \end{minipage}
  \begin{minipage}[b]{0.5\linewidth}
    \centering
    \includegraphics[scale=0.3]{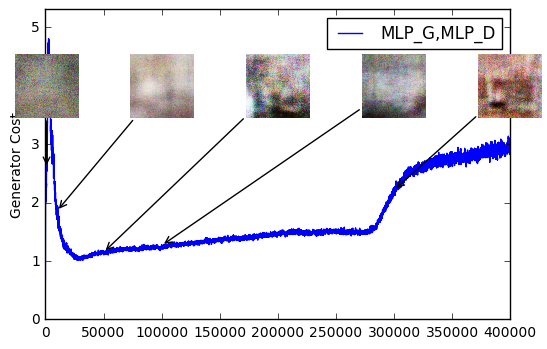}
  \end{minipage}
  \caption{Cost of the generator during normal GAN training,
for an MLP generator (upper left) and a DCGAN generator (upper right).
Both had a DCGAN discriminator. \textbf{Both curves
have increasing error}. Samples get better for the DCGAN
but the cost of the generator increases, pointing towards
no significant correlation between sample quality and loss. Bottom:
$MLP$ with both generator and discriminator. The curve goes
up and down regardless of sample quality. All training
curves were passed through the same median filter as in \autoref{fig::corr}.}
  \label{fig::gancorrgencost}
\end{figure}

\section{Sheets of samples} \label{app::sheets}
\newgeometry{top=0cm, bottom=0cm}

\begin{figure}[h]
  \centering
  \includegraphics[width=.7\linewidth]{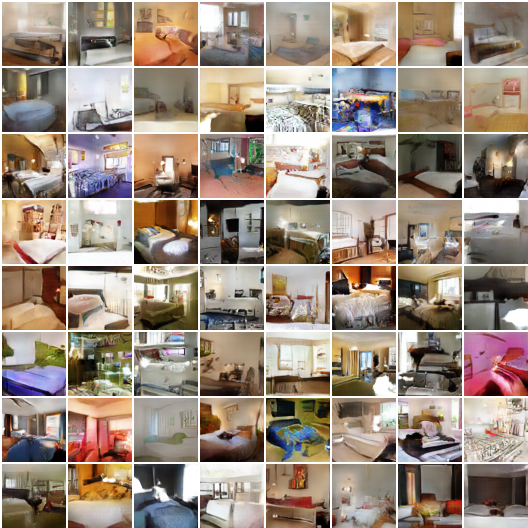}
  \caption{WGAN algorithm: generator and critic
are DCGANs.}
  \bigskip
\includegraphics[width=.7\linewidth]{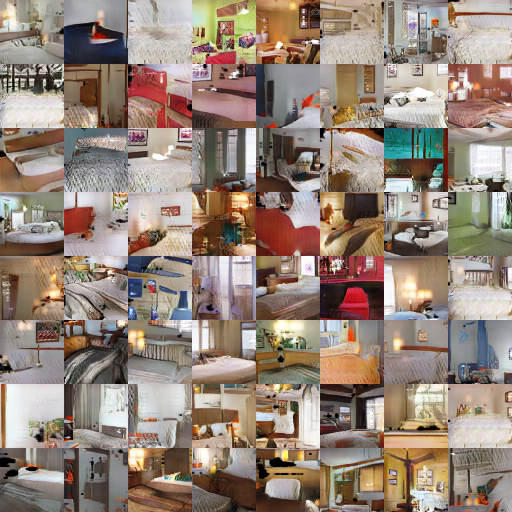}
  \caption{Standard GAN procedure: generator and discriminator
are DCGANs.}
  \label{fig::sheetdcgan}
\end{figure}

\begin{figure}[h]
  \centering
  \includegraphics[width=.7\linewidth]{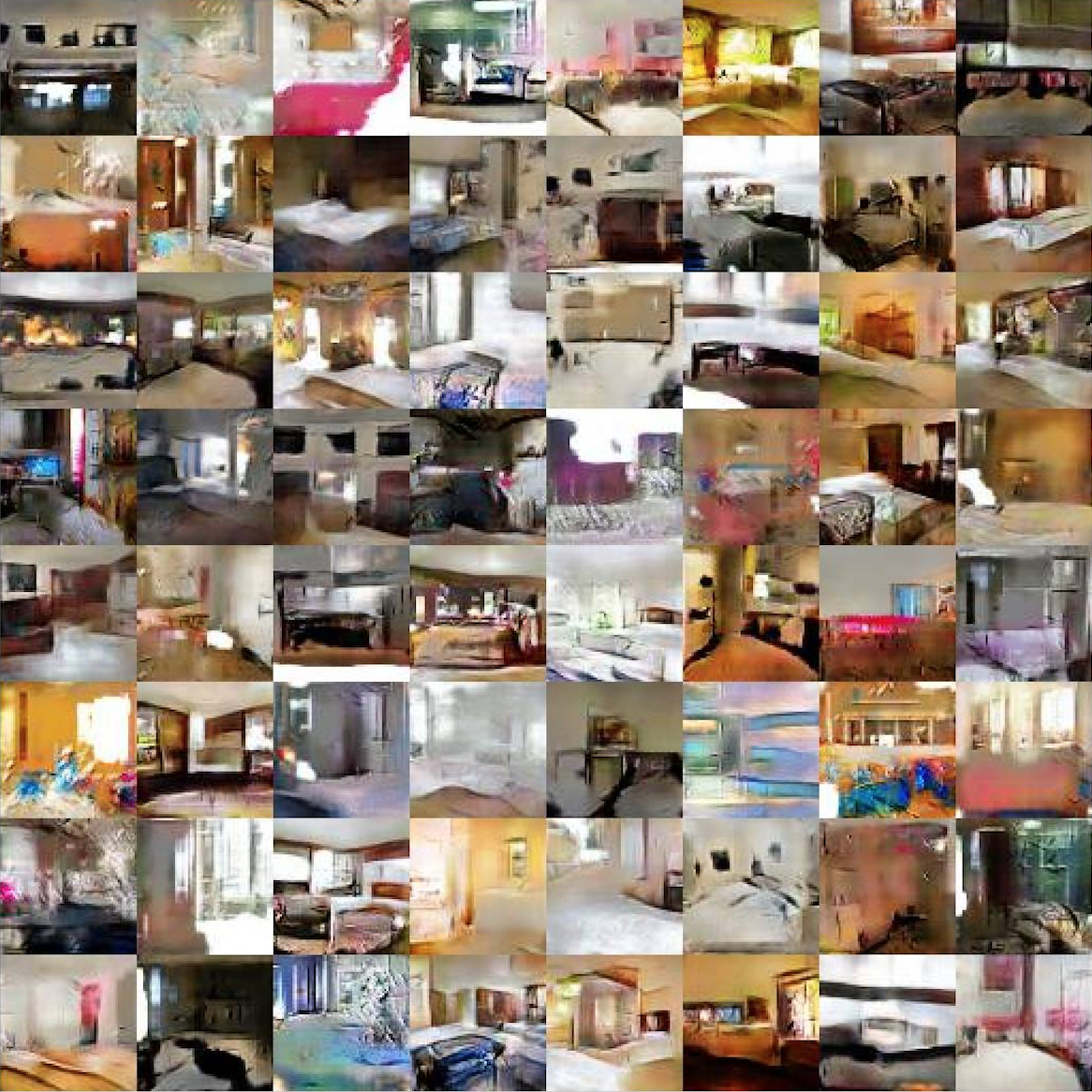}
  \caption{WGAN algorithm: generator is a DCGAN without
batchnorm and constant filter size. Critic is a DCGAN.}
  \bigskip
\includegraphics[width=.7\linewidth]{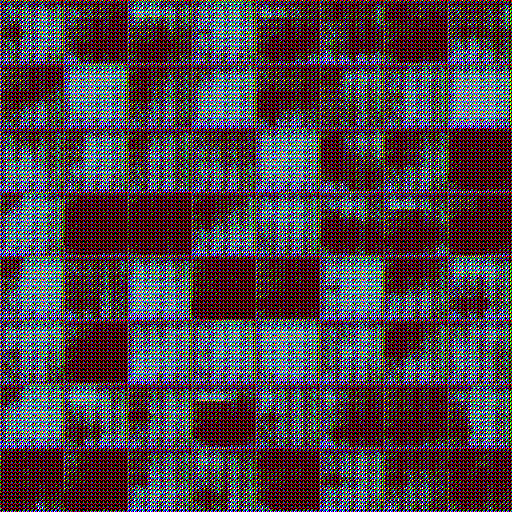}
  \caption{Standard GAN procedure: generator is a DCGAN without
batchnorm and constant filter size. Discriminator is a DCGAN.}
  \label{fig::sheetnobn}
\end{figure}

\begin{figure}[h]
  \centering
  \includegraphics[width=.7\linewidth]{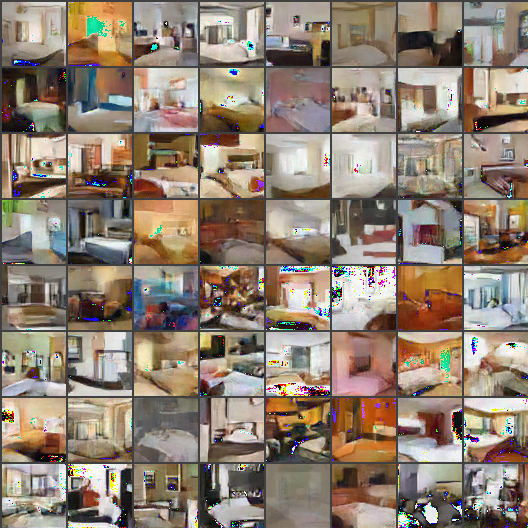}
  \caption{WGAN algorithm: generator is an MLP with 4 hidden
layers of 512 units, critic is a DCGAN.}
  \bigskip
\includegraphics[width=.7\linewidth]{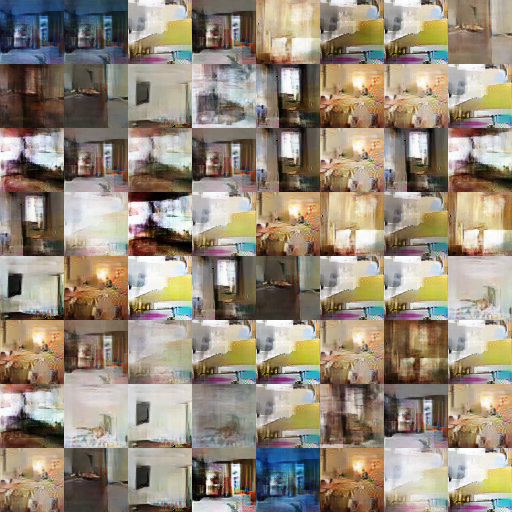}
  \caption{Standard GAN procedure: generator is an MLP with
4 hidden layers of 512 units, discriminator is a DCGAN.}
  \label{fig::sheetMLP}
\end{figure}

\restoregeometry

%\end{changemargin}

\end{appendices}

\end{document}